\RequirePackage{fix-cm}
\documentclass{svjour3}                     
\smartqed  
\usepackage{graphicx}

\usepackage{times}
\usepackage{soul}
\usepackage{url}
\usepackage[hidelinks]{hyperref}
\usepackage[utf8]{inputenc}
\usepackage[small]{caption}
\usepackage{booktabs}
\usepackage{pgfplots}
\usepackage{amsmath}
\usepackage{tikz}
\usepackage{amssymb}

\usepackage{nicefrac}

\usepackage[linesnumbered,ruled]{algorithm2e}
\SetKwInOut{Parameter}{parameter}
\usepackage{algpseudocode}
\usepackage{subcaption}

\newcommand{\act}[1]{#1}

\newcommand{\pref}{\succ}
\newcommand{\state}[1]{#1}

\newcommand{\Cp}{\ensuremath C}

\newcommand{\Yolobot}{\textsc{Yolobot}}
\newcommand{\omcts}{\mbox{O-MCTS}}
\newcommand{\omab}{\mbox{O-UCB}}
\newcommand{\ucb}{\mbox{UCB}}
\newcommand{\ohmab}{\mbox{OH-UCB}}
\newcommand{\multi}{\mbox{MultiSBM}}
\newcommand{\pbmcts}{\mbox{PB-MCTS}}

\newcommand{\mixmax}{\textsc{MixMax}}

\newcommand{\game}[1]{\textsl{#1}}

\usepackage{rotating}
\usepackage{multirow}
\usepackage{longtable}
\usepackage{colortbl}
\usepackage{booktabs}
\usepackage{url}
\newcommand{\citeauthor}[1]
{\cite{#1}}
\newcommand{\shortcite}[1]
{}
\newcommand{\citeyear}[1]
{}
\newcommand{\citeauthoryear}[1]
{}
\newcommand{\citet}[1]
{\citeauthor{#1}˜\shortcite{#1}}
\newcommand{\citep}{\cite}

\begin{document}


\title{Ordinal Monte Carlo Tree Search
}


\author{
Tobias Joppen         \and
        Johannes F\"urnkranz 
}


\institute{
}

\date{Received: date / Accepted: date}

\maketitle

\begin{abstract} 
In many problem settings, most notably in game playing, an agent receives a possibly delayed reward for its actions.
Often, those rewards are handcrafted and not naturally given.
Even simple terminal-only rewards, like winning equals $1$ and losing equals $-1$, can not be seen as an unbiased statement, since these values are chosen arbitrarily, and the behavior of the learner may change with different encodings. 
It is hard to argue about good rewards and the performance of an agent often depends on the design of the reward signal. 
%
In particular, in  
domains where states by nature only have an ordinal ranking and where meaningful distance information between game state values is not available,
a numerical reward signal is necessarily biased.
In  this paper we take a look at Monte Carlo Tree Search (MCTS), a popular algorithm to solve MDPs, highlight a reoccurring problem concerning its use of rewards, and show that an ordinal treatment of the rewards overcomes this problem.
Using the General Video Game Playing framework we show dominance of our newly proposed ordinal MCTS algorithm over other MCTS variants, based on a novel bandit algorithm that we also introduce and test versus \ucb.

\end{abstract}

\section{Introduction}
\label{sec:Introduction}

In reinforcement learning, an agent solves a Markov decision process (MDP) by selecting actions that maximize its long-term reward.
Most state-of-the-art algorithms 
assume numerical rewards.
In domains like finance, real-valued reward is naturally given,
but many other domains do not have a natural numerical reward representation.
In such cases, numerical values are often handcrafted by experts so that they optimize the performance of their algorithms. 
This process is not trivial, and it is hard to argue about good rewards.
Hence, such handcrafted rewards may easily be erroneous and contain biases.
For special cases such as domains with true ordinal rewards, it has been shown that it is impossible to create numerical rewards that are not biased. For example,
\cite{yannakakis2017ordinal} argue that emotions need to be treated as ordinal information. 

In fact, it often is hard or impossible to tell whether domains are real-valued or ordinal by nature.
Experts may even design handcrafted numerical reward without thinking about alternatives, since using numerical reward is state of the art and most algorithms need them.
In this paper we want to emphasize that numerical rewards do not have to be the ground truth and it may be worth-while for the machine learning community to have a closer look on other options, ordinal being only one of them.

A popular example where the use of numerical values fails is a minimalistic medicine treatment setting: 
Consider three possible outcomes of a treatment: healthy, unchanged and dead.
In the process of reward shaping one assigns a numerical value to each outcome.
Those numbers define the trade-off between how many patients have to be healed until one patient may die in comparison to no treatment (all patient unchanged).
It is impossible to avoid this trade-off with numerical values.
Using an ordinal score one could define the outcomes to be ordered as $\textit{healthy} > \textit{unchanged} > \textit{dead}$ without an implicit trade-off.
In this paper we present HO-UCB, an ordinal algorithm that is able to solve this treatment problem without trading off healed and dead patients.

\paragraph{MCTS}
Monte Carlo tree search (MCTS) is a popular algorithm to solve MDPs.
MCTS is used in many successful AI systems, such as AlphaGo \cite{silver2017mastering} or top-ranked algorithms in the general video game playing competitions \cite{perez2018general,YOLOBOT}.
A reoccurring problem of MCTS with limited time resources is its behavior in case of danger:
As a running example we look at a generic platform game, where an agent has to jump over deadly gaps to eventually reach the goal at the right.
Dying is very bad, and the more the agent proceeds to the right, the better.
The problem occurs by comparing the actions \emph{jump} and \emph{stand still}
:
jumping either leads to a better state than before because the agent proceeded to the right by successfully jumping a gap, or to the worst possible state (\emph{death}) in case the jump attempt failed. Standing still, on the other hand, safely avoids death, but will never advance to a better game state.
MCTS averages the obtained rewards gained by experience, which lets it often choose the safer action and therefore not progress in the game, 
because the (few) experiences ending with its death pull down the average reward of \emph{jump} below the mediocre but steady reward of standing still.
Because of this, the behavior of MCTS has also been called \emph{cowardly} in the literature \cite{jacobsen2014monte,khalifa2016modifying}.

Transferring those platform game experiences into an ordinal scale eliminates the need of meaningful distances. 
In this paper, 
we present an algorithm that only depends on pairwise comparisons in an ordinal scale, and selects \emph{jump} over \emph{stand still} if it more often is better than worse.
We call this algorithm Ordinal MCTS (\omcts{}) and compare it to different MCTS variants using the General Video Game AI (GVGAI) framework \cite{perez2016general}.

%

In the next section we introduce MABs and MDPs as our problem definitions and \ucb, \multi, \pbmcts{} and MCTS as already known solutions to solve those problems.
In Section~\ref{sec:O-MCTS}, we present our novel algorithms, followed by experiments to present how our algorithms compare to the existing ones (Sections~\ref{sec:Setup} and~\ref{sec:Experiments}).
\
\section{Monte Carlo Tree Search}
\label{sec:Foundations}
In this section, we briefly recapitulate Monte Carlo tree search and some of its variants, which are commonly used for solving Markov decision processes (MDP).
Prior to this, we introduce the multi-armed bandit problem (MAB) and a common used algorithm to solve MABs: UCB1.

\subsection{Multi-Armed Bandits}
\label{sec:MABs}
A \emph{Multi-Armed Bandit} (MAB) is a common problem one often faces:
One has a set of possible actions (or arms, $A$) to choose one from.
Once an arm $a$ is chosen, the player receives a reward $r$ sampled from a unknown distribution $X_a$. 
Often rewards are designed to be numerical ($r \in R$).

The task is to identify the optimal arm, which returns the best rewards.
For numerical rewards, the \emph{best arm} $a^*$ can be defined to be the arm with the highest rewards in average, or $a^* = \arg\max_a \mathbb{E}[X_a]$.
For ordinal rewards, the reward function does not necessarily be in $\mathbb{R}$, but in $\mathbb{O}$, where all elements $o \in \mathbb{O}$ can be ordered via a given preference function $o_1 < o_2$.
Notice, that for $\mathbb{R}$ such a preference function is given through the natural ordering.
Since one can not add together or calculate an average of elements from $\mathbb{O}$, we need a different definition of what the \emph{best arm} is.
Other than for numerical rewards, there is not such a consensus about what the definition of optimality.
In this paper, we are interested in the borda winner \cite{black1976partial} - the arm that has the highest chance of beating an other randomly chosen arm.
We present our algorithm to solve ordinal MABs in a later section.

\subsection{Markov Decision Process}
Conventional Monte Carlo tree search assumes a scenario in which an agent moves through a state space by taking different actions.
This can not be modeled by a MAB, since the arms and their reward distributions are fixed.
A Markov Decision Process (MDP) \cite{MDP} takes MABs to a next level by not only being in one fixed state with its fixed arms, but by changing the state every time an action is played.
A MDP can be formalized as the following:


\begin{itemize}
	
	\item[--] A (finite) set of \emph{states} $S$. 
	
	\item[--] A (finite) set of \emph{actions} $A$ the agent can perform. Sometimes, only a subset $A(\state{s}) \subset A$ of actions is applicable in a state $\state{s}$.
	
	\item[--] A Markovian \emph{state transition} function $\delta(\state{s}'\mid \state{s},\act{a})$ denoting probability that invoking action $\act{a}$ in state $\state{s}$ leads the agent to state $\state{s'}$.
	
	\item[--] A \emph{reward function} $r(\state{s}) \in \mathbb{R}$ that defines the reward the agent receives in state $\state{s}$.
	
	
	\item[--] A distribution of \emph{start states} $\mu(\state{s}) \in [0,1]$, 
	 defining the probability that the MDP starts in that state. We assume a single start state $\state{s}_0$, with $\mu(\state{s}_0)=1$ and $\mu(\state{s}) = 0 \; \forall \state{s} \neq \state{s}_0$.
\item[--] A set of \emph{terminal states} for which $A(\state{s}) = \emptyset$. We assume that only terminal states are associated with a non-zero reward.

\end{itemize}

The task is to learn a \emph{policy} $\pi(\act{a}\mid \state{s})$ that defines the probability of selecting an action $\act{a}$ in state $\state{s}$. The optimal policy $\pi^*(\act{a}\mid \state{s})$ maximizes the expected, cumulative reward
\begin{equation}
\begin{aligned}
V(\state{s}_t) &= \mathbb{E} \left[ \sum_{t=0}^\infty \gamma^t r(\state{s}_t) \right],\\
&= r(\state{s}_t) + \gamma \int_S \int_A \delta(\state{s}_{t+1}\mid \act{a}_t,\state{s}_t)\pi(\act{a}_t\mid \state{s}_t) V(\state{s}_{t+1}). 
\end{aligned}
\label{eq:value}
\end{equation}
The optimal policy maximizes $V(\state{s}_t)$ for all time steps $t$. Here, $\gamma \in [0,1)$ is a discount factor, which dampens the influence of later events in the sequence.
For finding an optimal policy, one needs to solve the so-called exploration/exploitation problem. The state/action spaces are usually too large to sample exhaustively. Hence, it is required to trade off the improvement of the current, best policy (exploitation) with an exploration of unknown parts of the state/action space.

We also investigate an ordinal variation of the classical MDP: O-MDP \cite{weng2011markov}
Just like it is for the ordinal MAB, the ordinal MDP does use ordinal rewards and thus numerical definitions of optimality and regret can not be applied here.
As described in the last section, we are interested in the borda winner in case of ordinal rewards.
The borda winner maximizes the chance of beating the other arms.
The chance for an arm $a$ to beat all other arms is called borda score ($B(a)$ - see Section~\ref{sec:Borda}).
Each non-optimal arm has a lower borda score than the optimal arm.
The regret of playing a non optimal arm $a$ instead of $a^*$ is the difference of borda score: $\textit{regret}_a = B(a^*) - B(a)$.
Obviously it is $\textit{regret}_{a^*}= 0$.
Note that it is not possible to use one numerical bandit to optimize the borda score, since it is not the direct reward that is visible for the agent.
Further more, the borda score is not only dependent on the distribution of the current arm, but also is dependent on the reward distributions of all other arms (since the borda score is defined on comparisons of those).


\subsection{Multi-armed Bandits}
After introducing the problem frameworks MAB and MDP, we introduce common algorithms to solve those.
We start with a popular algorithm to solve MABs, where the task is to identify the arm (or action) with the highest return by repeatedly pulling one of the possible arms.
In this setting, there is only one non-terminal state, and 
the task is to achieve the highest average reward by playing theoretically infinite times.
Here the exploration/exploitation dilemma occurs because the player must play the best-known arm trying to maximize the average reward (exploitation), but also needs to search for the best arm among all alternatives (exploration).
A well-known technique for resolving this dilemma in bandit problems is the \emph{upper confidence bounds (UCB)} \cite{UCB}.
UCB estimates upper bounds on the expected reward for a certain arm, and chooses the action with the highest associated upper bound. We then observe the outcome and update the bound. The simplest UCB policy (\ref{eq:ucb})
gives a bonus $\sqrt{\frac{2 \ln n}{n_j}}$ to the average reward $\bar{X_j}$, which depends on the number of visits, thereby increasing the chance that arms that have not yet been frequently played are selected in subsequent iterations.
\begin{equation}
UCB1 = \bar{X_j} + \sqrt{\frac{2 \ln n}{n_j}}
\label{eq:ucb}
\end{equation}

\noindent
The first term favors arms with high payoffs, while the second term guarantees exploration \cite{UCB}.  The reward is expected to be bound by $[0,1]$.
In Section~\ref{sec:O-MABAlg}, we introduce two algorithms that are able to solve ordinal MABs.

\subsection{Duelling Bandits}
A related topic to ordinal bandits are dueling bandits, where at each time step two arms are pulled and the agent gets one reward indicating which arm won in the direct comparison. 
In reality, repeated dueling of different pairs is often used to identify the best element of a set, like in most sport leagues where different teams battle each other.
The biggest downside of this approach (at least from a energy or optimization point of view) is that each team needs to play against each other team at least once to be able to identify a winner.
If it would be possible to measure teams independently, only one measure per team would suffice to identify the winner.
Sample efficiency is the biggest difference between dueling and ordinal bandits:
In ordinal bandits, it is possible to measure the quality of one action, where for dueling bandits it always needs two actions to be compared against each other.

It has been shown that it is possible to reduce dueling bandits to common numerical bandits\cite{ailon2014reducing}.
Hence it is possible to optimize a bandit using preference information by relying on numerical bandit algorithms.
In the following, we intoduce the MultiSBM algorithm which enables to compare UCB and our ordinal bandit algorithm to a dueling bandits approach.

\subsection{MultiSBM}
The MultiSBM algorithm is able to learn from preference feedback, while not being restricted to the dueling bandit framework, where two arms are pulled at a time instead of one.
The main idea is to have a numerical bandit for each arm (for example using UCB).
In each round $t$ the last round played arm $a_{t-1}$ defines which bandit will be used to select the arm to be played $a_t$.
The feedback signal for the current bandit is the preference information of how $a_{t-1}$ did perform in comparison to $a_t$\cite{ailon2014reducing}.
%
We use MultiSBM as a preference-learning baseline to compare to UCB and our ordinal algorithms.

\subsection{Monte Carlo Tree Search}
\label{sec:MCTS}
Monte Carlo tree search (MCTS) is a method for approximating an optimal policy for a MDP.
It builds a partial search tree, which is more detailed where the rewards are high.
MCTS spends less time evaluating less promising action sequences, but does not avoid them entirely 
to explore the state space.
The algorithm iterates over four steps \cite{MCTSSurvey}:

\begin{enumerate}
\item \emph{Selection:} Starting from the root node $v_0$, a \emph{tree policy} traverses to deeper nodes $v_k$, until a state with unvisited successor states is reached.
\item \emph{Expansion:} One successor state is added to the tree.
\item \emph{Simulation:} Starting from the new state, a so-called \emph{rollout} is performed, i.e., random actions are played until a terminal state is reached or a depth limit is exceeded.
\item \emph{Backpropagation:} The reward of the last state of the simulation is backed up through all selected nodes.
\end{enumerate}
The UCT formula
\begin{equation}
a^*_v = \max_{a \in A_v} \bar{X}_v(a) + 2 \Cp \sqrt{\frac{2 \ln n_v}{n_{v}(a)}}
\label{eq:uct}
\end{equation}
\noindent
has been derived from the UCB1 algorithm (see \ref{eq:ucb} and is used to select the most interesting action $a^*_v$ in a node $v$ by trading off the expected reward estimated as $\bar{X}_v(a) =$ $\sum_{i=0}^{n_v}X_v^{(i)}(a)/n_v(a)$ from $n_v(a)$ samples $X_v^{(i)}(a)$, 
with 
an exploration term $\sqrt{2 \ln (n_v)/n_v(a)}$. The trade-off parameter $\Cp$ is often set to $\Cp = 1/\sqrt{2}$, which has been shown to ensure convergence for rewards $\in [0,1]$ \cite{UCT}.
In the following, we will often omit the subscript $v$ when it is clear from the context.

\subsection{Preference-Based Monte Carlo Tree Search}
\label{sec:PB-MCTS}

\begin{figure}[t]
\centering
\includegraphics[width=0.9\columnwidth]{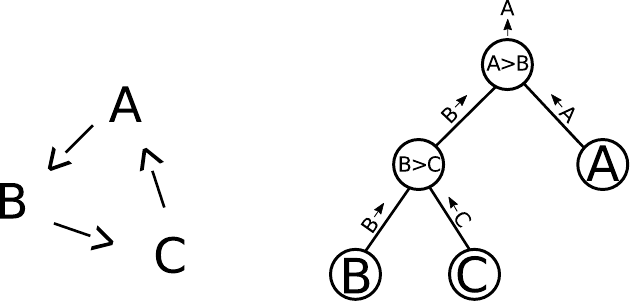}
\caption{Three nontransitive actions. The tree introduces a bias to solve nontransitivity.}
\label{fig:transitivity}
\end{figure}

A version of MCTS that uses preference-based feedback (\pbmcts{}) was recently introduced by \citet{joppen2018preference}. 
In 
this setting, agents receive rewards in form of preferences over states.
Hence, feedback about single states $s$ is not available, it can only be compared
to another state $s'$, i.e., $s\succ s'$ (dominance), $s'\succ s$, or $s\not\sim s'$ (incomparable).

An iteration of \pbmcts{} contains the same abstract steps like MCTS, but their realization differs.
First and foremost, it is impossible to use preference information on a vanilla MCTS iteration, since it only samples a single trajectory, whereas a second state
 is needed for a comparison.
Hence, \pbmcts{} does not select a single path per iteration but an entire subtree of the search tree. 
In each of its nodes, 
two actions are selected 
that can be compared to each other.
For the selection step, 
a modified version of the dueling bandit algorithm RUCB \cite{zoghi14} is used to select actions.

%
%

There are two 
main disadvantages with this approach:
\begin{enumerate}
	\item \emph{No transitivity} is used. Given ten actions, MCTS needs only $10$ iterations to have a first estimation of quality of each action. In the preference-based approach, each action has to be compared with each other action until a first complete estimation can be done. These are $(10 \cdot 9)/2 = 45$ iterations, i.e., in general the effort is quadratic in the number of actions. 
	\item A \emph{binary subtree} is needed to learn on each node of the currently best trajectory. Instead of a path of length $n$ for vailla MCTS, the subtree consists of $2^n-1$ nodes and $2^{n-1}$ trajectories instead of only one, 
	causing an 
	exponential blowup of \pbmcts{}'s search tree.
\end{enumerate}
Hence, we believe that \pbmcts{} does not make optimal use of available computing resources, since on a local perspective, transitivity information is lost, and on a global perspective, the desired asymmetric growth of the search tree is undermined by the need for selecting a binary tree.
Note that even in the case of a non-transitive domain, \pbmcts{} will nevertheless 
obtain a transitive policy, as illustrated in Figure~\ref{fig:transitivity}, where the circular preferences between actions \textsf{A}, \textsf{B}, and \textsf{C} can not be reflected in the resulting tree structure. 

\section{Ordinal Algorithms}
In this section we introduce three novel ordinal algorithms, two for the multi-armed bandit setting and one novel ordinal MCTS algorithm for markov decision processes.

\subsection{Borda Bandit}
\label{sec:O-MABAlg}
Our first bandit algorithm \omab{} is the base for the following algorithms and a very core contribution of this paper.
The main idea is to store all ordinal reward values per action and evaluate how good an action is depending on how probably it is that a random ordinal reward of those seen per arm is better than a random one of the other actions.
In the following we describe this idea in more detail.

\subsubsection{The Borda Score}
\label{sec:Borda}


The Borda score is based on the Borda count which has its origins in voting theory \cite{black1976partial}. 
Essentially, it estimates the probability of winning against a random competitor. 
In our case, $B(a)$ is the probability of action $a$ to win against any other action $b \neq a$ available (with tie correction):
$$ B(a) =  \frac{1}{| A|-1}\sum_{b\in A \setminus \{a\}} \Pr(a\succ b)$$
where
$$ \Pr(a\succ b) = \Pr(X_a > X_b) + \nicefrac{1}{2} \Pr(X_a = X_b)$$
and $X_i$ is the random variable responsible for sampling the ordinal rewards for arm $i$.
The true value for $\Pr(X_a > X_b)$ and $\Pr(X_a = X_b)$ is unknown but can be estimated empirically.
In dueling bandits or PB-MCTS, one picks two arms $a_i$ and $a_j$ and receives a direct sample for $\Pr(X_{a_i} > X_{a_j})$.
In this paper, we take a different approach by assuming availability of ordinal rewards per arm.
%
For each available action $a \in A$, we store the empirical density function $\hat{f}_{a}$ and the empirical cumulative density function $\hat{F}_{a}$ of all backpropagated ordinal rewards $o \in O$.
Given those functions, we can estimate $\Pr(X_a \succ X_b)$\citep{vargha}:
\begin{equation}
\begin{aligned}
\label{formula:ACalc}
&\hat{B}(a\succ b) \\
&=\int \Pr[X_{a} < o] \hat{f}_{b}(o) + \nicefrac{1}{2} \Pr[X_a = o] \hat{f}_{b}(o) do  \\
& = \sum_{i=2}^{|O|} \frac{\hat{F}_{a}(o_{n-1})+\hat{F}_{a}(o_n)}{2} \hat{f}_{b}(o_n)
\end{aligned}
\end{equation}

\subsubsection{Algorithm}
In comparison to \ucb, \omab{} uses the Borda score $\hat{B}$ as an exploitation term to choose the arm to play:
\begin{equation}
a^* = \max_{a \in A} \hat{B}(a) + 2 \Cp \sqrt{\frac{2 \ln n}{n(a)}}.
\label{eq:o-mab}
\end{equation}

To wrap it up there are two differences to the UCB implementation:
The selection of which arm to play uses a different exploitation term: $\hat{B}$; and instead of  updating a running average per action, here $\hat{f}_{a}$ and $\hat{F}_{a}$ gets updated according to the new ordinal reward.
We present bounds on the regret for this choice of actions in bandit algorithms:
\begin{theorem}For all $K>1$ and $O>1$, the expected regret after any number $n$ of plays of an ordinal bandit using (\ref{eq:o-mab}) is at most $\mathcal{O}(\sum_{i\in A^-}\big( \frac{\ln n}{\Delta_i}\big))$ where $A^-$ is the set of non optimal arms.
\end{theorem}

\begin{proof}See appendix.\footnote{The proof can be found at \url{http://tiny.cc/OMCTS_proof}}
	\end{proof}

\subsection{Hierarchical Borda Bandit}
\label{sec:OH-MAB}
The Borda Bandit \omab{} will favour to play that arm, which beats the other arms most often.
In the most simple setting of two arms, \omab{} will choose the arm that more often wins direct duels with the other arm, independent of how much the difference between the outcomes is.
In a lottery setting where one can choose to play or not to play, \omab{} will choose to play only if the chances of winning the lottery are above $50\%$.
Same for the medicine treatment: If you can choose between no treatment or treatment, \omab{} will choose treatment if the treatment was successful in more than $50\%$ of the cases.
Having people die from the treatment in $49\%$ would be fine for \omab{}, what is understandable since every distance measure is removed and there is no reason for \omab{} to argue that death is worse than a good treatment.
But it might be desirable to define some notion of distance, like death is much worse than everything else.
The following bandit algorithm \emph{Hierarchical Ordinal UCB}(\ohmab) provides parameters to define hierarchical preferences and thus different distance measures.

A \ohmab{} is parameterized by a tuple $(\bar{z}, H)$, where $\bar{z}$ is a critical value to check for significance (needed later), and a $d$-sized hierarchy of ordinal values $H \in \mathcal{P}(O)^d$ where $\mathcal{P}(O)$ refers to the power set of $O$.
In each hierarchy level $h \in \mathcal{P}(O)$ multiple elements of $O$ can be selected or not.
The function $M_h \in O\times O$ maps all elements of $O$ to their next bigger element in $O$ which is selected in $h$ or, in case of no bigger selected element existing, the highest element of $O$.
The selected elements now define a notion of distance:
Two elements $o$ and $o'$ are close-by in a hierarchy $h$, if $M_h(o) = M_h(o')$, or in other words are close by if there is no selected element between $o$ and $o'$ in $h$.
The main idea of \ohmab{} is the usage of one \omab{} agent per hierarchy $h$, where each agent $a_h$ perceives an ordinal reward signal $o$ as $M_h(o)$.

Selecting an arm $a\in A$ is done in an iterative manner going from the first $i=1$ to last hierarchy $i=d$.
The set of valid arms is initialized with all possible arms $\bar{A}=A$.
The current agent of hierarchy $h_i$ is used to identify the best action $a^i \in \bar{A}$ of all valid arms (picking the one that maximizes $b(a)$).
Now, each other valid action $\hat{a}\in \bar{A}\setminus \{a^i\}$ is tested whether $a^i$ is significant better than $\hat{a}$ using $\bar{z}$ as the critical value in a Mann-Whitney U test between those two arms when ignoring the exploration trade off:
Having $n^i$ and $\hat{n}$ being the amount of plays of arms $a^i$ and $\hat{a}$, the Mann-Whitney U value for $a^i$ and $\hat{a}$ can be derived using the borda score $\hat{B}(a^i\succ \hat{a})$:
$U_{a^i} = \hat{B}(a^i\succ \hat{a})\cdot n^i \hat{n}$ and $U_{\hat{a}} = \hat{B}(\hat{a}\succ a^i)\cdot n^i \hat{n}$ \cite{sprinthall1990basic}.
For large enough samples (we use $n^i > 3$, $\hat{n}>3$ and $n^i + \hat{n}>20$) U is approximately normally distributed and hence we can check for significance by testing if
\begin{equation}
\begin{aligned}
\label{formula:significanceHierarchy}
\hat{z}& > \frac{U_{a^i} - m_{U_{a^i}}}{\sigma_{U_{a^i}}}\\
\hat{z}& > \frac{U_{a^i} - \frac{n^i \hat{n}}{2}}{\sqrt{\frac{n^i \hat{n} (n^i+ \hat{n}+1}{12}}}
\end{aligned}
\end{equation}
where $m$ and $\sigma$ are the mean and standard deviation of $U$\citep{vargha,sprinthall1990basic}.

If a valid arm $\hat{a}$ is significant worse that the best arm $a^i$, it is removed from the list of valid arms $\bar{A} = \bar{A}\setminus \{\hat{a}\}$.
If, after each valid arm has been tested for significance, $a^i$ is the only valid arm left, it is returned as the arm to be played by \ohmab.
Otherwise, the next hierarchy $i+1$ continues.
If no hierarchy is left ($i = d$), $a^i$ is returned as to be played, too.
After a arm $a$ has been played and a ordinal reward $o$ has been received, each of the bandits $a_h$ is updated with $M_h(o)$.

To give further insight and motivation, we inspect on how \ohmab{} could be used for the medicine treatment setting:
The main problem for \ucb in this setting is to avoid \emph{dead patients} without defining a clear trade off.
Here, we can use the first hierarchy level to do exactly that by only selecting \emph{dead patients} in $h_1$.
Hence, this bandit only perceives two rewards: \emph{patient dead} and \emph{patient not dead} and will therefore favor those actions that lead to the least amount of dead patients.
Depending on the samples seen and $\hat{z}$ this first hierarchy filters out those arms with significant more dead patients.
The second and last hierarchy would select the complete set of ordinal values to allow the most fine grained level of optimization.


\subsection{Ordinal Monte Carlo Tree Search}
\label{sec:O-MCTS}
In this section, we introduce \omcts, an MCTS variant which only relies on 
ordinal information to learn a policy.
We derive an ordinal MCTS algorithm by using \omab{} instead of \ucb{} as the tree policy.

\emph{Ordinal Monte Carlo tree search} (\omcts) proceeds like conventional MCTS as introduced in Section~\ref{sec:MCTS}, but 
replaces the average value $\bar{X}_v(a)$ in~\eqref{eq:uct}
 with the Borda score $B_v(a)$ of an action $a$ in node $v$, where $v$ now defines the current state or node in the tree.
Here, each arm is rated according to its mean performance against the other arms.
To our knowledge, Borda score has not been used in MCTS before, even tough several papers have investigated its use in dueling bandit algorithms \cite{DuelingBandits,urvoy2013generic,jamieson2015sparse}.
To calculate the Borda score for each action in a node, \omcts{} stores the backpropagated ordinal values, and estimates pairwise preference probabilities $P_v(a\pref b)$ from these data.
Hence, it is not necessary to do multiple rollouts in the same iteration as in \pbmcts{} because
current rollouts can be directly compared to previously observed ones.

Note that $P_v(a \succ b)$ can only be estimated if each action was visited at least once.
Hence, similar to other MCTS variants, we enforce this by always selecting non-visited actions in a node first.

\subsection{Discussion}
Although the changes from MCTS to \omcts{} are comparably small, the algorithms 
have very different characteristics.
In this section, we highlight some of the differences between \omcts{} and MCTS.
\paragraph{Different Bias}
As mentioned previously, MCTS has been blamed for behaving cowardly, by preferring
safe but unyielding actions over actions that have some risk but will in the long
run result in higher rewards.
As an example, consider Figure~\ref{fig:loss}, which shows in its bottom row the distribution of trajectory values for two actions over a range of possible rewards.
One action (circles) has a mediocre quality with low deviation, whereas
the other (stars) is sometimes worse but often better than the first one.
Since MCTS prioritizes the stars only if the average is above the average of circles, MCTS would often choose the safe, mediocre action.
In the literature one can find many ideas to tackle this problem, like \mixmax{} backups \citep{khalifa2016modifying} or adding domain knowledge (e.g., by giving a direct bonus to actions that should be executed \cite{perez2018general,YOLOBOT}).
\omcts{} takes a different point of view, by not comparing average values but by comparing how often stars are the better option than circles and vice versa. As a result, it would prefer the star action, which is preferable in 70\% of the games.
Please note that the given example can be inverted such that MCTS takes the right choice instead of \omcts{}.

\begin{figure}[t]
\centering
\includegraphics[width=\columnwidth]{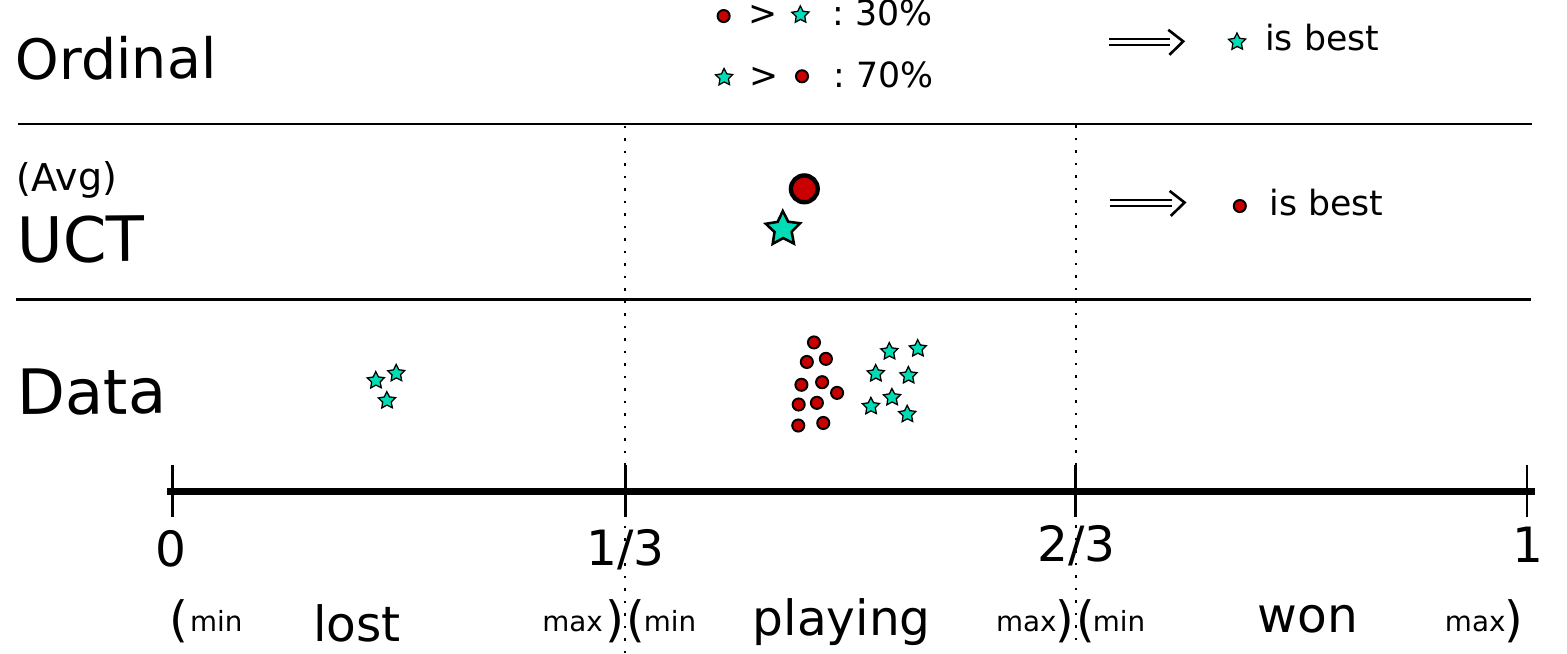}
\caption{Two actions with different distributions.}
\label{fig:loss}
\end{figure}

\paragraph{Hyperparameters and Reward Shaping}
When trying to solve a problem with MCTS (and other algorithms, too), rewards can be seen as hyperparameters that can be tuned manually to make an algorithm work as desired.
In theory this can be beneficial since you can tweak the algorithm with many parameters. 
In practice it can be very painful since there often is an overwhelming number of hyperparameters to tune this way.
This tuning process is called \emph{reward shaping}.
In theory, one can shape the state rewards until a greedy search is able to perform optimal on any problem.

\omcts{} reduces the number of hyperparameters by only asking for ordinal rewards; which is like asking for a ranking of states instead of real numbers for each state.
This limits the possibilities of reward shaping but induces a fixed bias using the borda method.

\paragraph{Computational Complexity}
Clearly, a na\"ive computation of $\hat{B}$ is computationally more expensive than MCTS' calculation of a running average.
We hence want to point out that once a new ordinal reward is seen it is possible to incrementally update the current value of $\hat{B}$ instead of calculating it again from scratch. 
In our experiments, updating the Borda score needed $3$ to $20$ times more time than updating the average (depending on the size of $O$ and $A$).
These values do only show the difference in updating $\hat{B}$ in comparison to updating the running average, not the complete algorithms (where the factor is much lower, mostly depending on the runtime of the forward model).

\section{Experimental Setup}
\label{sec:Setup}

The experiments are split in two major sections: Bandits and Tree Search.
To be able to compare numerical baseline algorithms with ordinal algorithms, we first define numerical rewards and derive ordinal rewards from there.

\subsection{Bandit Setup}
For the bandit setup, we inspect the medicine treatment problem from above:
Imagine patients can be dead or alive after a medical treatment.
In case of being alive, a continuous scale of wellbeing can differentiate the quality of the treatment, while being dead is the worst possible outcome.
This can be modeled in a numerical score by defining death as a reward of $0$ and everything else in $(0,1]$.
The aim is to identify a good treatment with as less dead patients as possible.
This on purpose is defined to not look for any dead-patient/wellbeing trade-off.
There is a clear preference to minimize dead patients at first and then maximize the wellbeing score.
We want to emphasise that this optimization can not be modeled as a numerical score but is of a higher dimension.
Hence UCB1 will not be able to find that optimum but will only be able to maximize the expected numerical reward.
Furthermore, O-UCB and MultiSBM will also not directly minimize the amount of death patients too, since they will maximize the pull of the arm that beats the other arms in average.
Nevertheless this setup can prove two main statements:
First, \ohmab{} can maximize the wellbeing score while minimizing the amount of deaths without defining a numerical trade off between those targets and 
second, \omab{} can optimize the wellbeing score directly. 
The latter optimization can be compared to \ucb and \multi{} in terms of convergence speed.

Notice that we use a parameterized version of all bandit algorithms with parameter $c$, which trades off exploration and exploitation (compare Formula~\ref{eq:ucb} and Formula~\ref{eq:uct}).

In our bandit setting, we have the four different treatments available (four actions):
\begin{itemize}
    \item A $20\%$ death (r=$0$) / $80\%$ maximal wellbeing (r=$1$) treatment (a good but maybe killing treatment)
    \item A $80\%$ death (r=$0$) / $20\%$ maximal wellbeing (r=$1$) treatment (a very risk and worse treatment than the one above)
    \item A no-treatment baseline (r=$0.6$)
    \item A treatment that slightly increases the level of wellbeing of patients (r=$0.7$)
\end{itemize}

\noindent
Using this bandit, we test four different algorithms: \omab, \ohmab, \ucb{} and \multi.
Each agent has 500 action pulls available per experiment, with each experiment being averaged over 20 runs.
We repeat each experiment with the parameter $c\in(0.1,0.2, 0.4, 0.6, 0.8, 1, 1.2, 1.4, 1.6, 1.8, 2)$ and compare the best results of each agent.
\omab, \ohmab{} and \multi{} can derive their preferences from the wellbeing score described above. 
As described in Section~\ref{sec:OH-MAB}, two hierarchies are used for \ohmab, the first to define that \emph{dead} is worse than everything else (or $v=0 << v>0$) and a full second hierarchy, where all ordinal values are considered. 
Hence \ohmab{} should be able to prioritize those actions that do not lead to death and optimizes the remaining actions as \omab{} would do, too. 
We use $\hat{z}=0.65$ as the critical z value to check for significance in \ohmab.

\subsection{Bandit Results}
The target is to minimize dead patients. In cases of ties, the average wellbeing score is used as a tie breaker.
First, we inspect Figure~\ref{fig:bandits1}, where the amount of dead patients is shown over time for the best $c$ parameters per algorithm. 
It can clearly be seen, that \ohmab is the only algorithm that converges towards a low number of deaths.
This is to no surprise, since the other algorithms maximize the wellbeing value directly.

A plot showing the best wellbeing values per time is shown at Figure~\ref{fig:bandits2} for the $c$ parameters that increased this score the best.
Here one can see, that O-UCB as well as UCB are able to maximize the value.
MultiSBM is a bit behind, since it needs to learn from pair of actions and is not able to abstract information between arm pairs as good as \ucb{} or \omab{} can do.
In contrast to MultiSBM, O-UCB is able to use a single feedback value of one arm and can compare it to any other value of any other arm.

Lastly, we present an overview showing the influence of parameter $c$ in Table~\ref{tbl:banditResults} by displaying the average wellbeing score and amount of deaths per algorithm and parameter.
Inspecting $c\geq 0.2$ there is an interesting different between \ucb{} and \omab{} shown:
\omab{} keeps the amount of deaths at around 83 with a decrease of wellbeing when increasing exploration.
In contrast to that, \ucb{} reduces both, wellbeing score and amount of deaths.



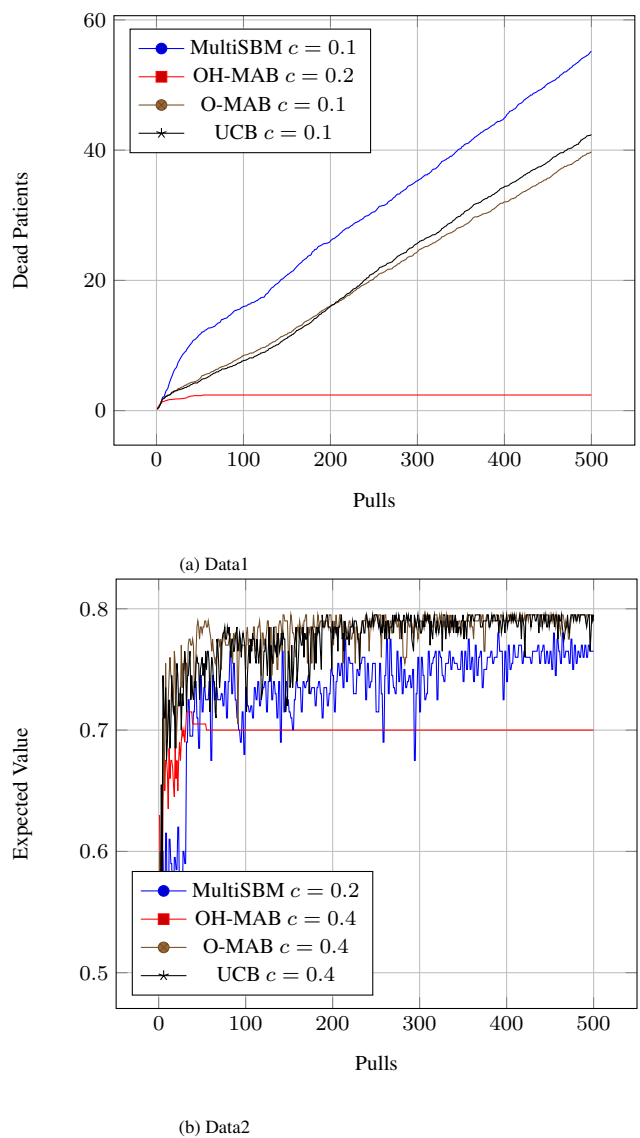
\begin{figure}[htb]
	\centering
	\begin{subfigure}{0.45\textwidth}
		\centering
		\begin{tikzpicture}
  \begin{axis}[
    xlabel={Pulls},
    ylabel={Dead Patients},
    grid=major,
    legend pos=north west
  ]
\addplot table [x=X, y=0S_MultiSBM_0.1, mark=none, col sep=semicolon, /pgf/number format/read comma as period] {banditsReport.csv}; \addlegendentry{MultiSBM $c=0.1$}
%
\addplot table [x=X, y=0S_O_H_MAB_0.2, mark=none, col sep=semicolon, /pgf/number format/read comma as period] {banditsReport.csv}; \addlegendentry{OH-MAB $c=0.2$}
\addplot table [x=X, y=0S_O-MAB_0.1, mark=none, col sep=semicolon, /pgf/number format/read comma as period] {banditsReport.csv}; \addlegendentry{O-MAB $c=0.1$}
%
\addplot table [x=X, y=0S_UCB_0.1, mark=none, col sep=semicolon, /pgf/number format/read comma as period] {banditsReport.csv}; \addlegendentry{UCB $c=0.1$}
%
\end{axis}
\end{tikzpicture}
		\caption{Data1}
		\label{fig:bandits1}
	\end{subfigure}
	
	\begin{subfigure}{0.45\textwidth}
		\centering
		\begin{tikzpicture}
  \begin{axis}[
    xlabel={Pulls},
    ylabel={Expected Value},
    grid=major,
    legend pos=south west
  ]
\addplot table [x=X, y=E_MultiSBM_0.2, mark=none,smooth , col sep=semicolon, /pgf/number format/read comma as period] {banditsReport.csv}; \addlegendentry{MultiSBM $c=0.2$}
%
\addplot table [x=X, y=E_O_H_MAB_0.2, mark=none,smooth , col sep=semicolon, /pgf/number format/read comma as period] {banditsReport.csv};	\addlegendentry{OH-MAB $c=0.4$}
%
\addplot table [x=X, y=E_O-MAB_0.4, mark=none,smooth , col sep=semicolon, /pgf/number format/read comma as period] {banditsReport.csv}; \addlegendentry{O-MAB $c=0.4$}
%
\addplot table [x=X, y=E_UCB_0.4, mark=none,smooth , col sep=semicolon, /pgf/number format/read comma as period] {banditsReport.csv}; \addlegendentry{UCB $c=0.4$}
%
\end{axis}
\end{tikzpicture}
		\caption{Data2}
		\label{fig:bandits2}
	\end{subfigure}
	\caption{Bandit results}
\end{figure}

\begin{table}[htbp]
\centering
\tabcolsep=0.11cm
	\begin{tabular}{c||rl|rl|rl|rl}
		\begin{sideways}$C$ Parameter\end{sideways}
		&\multicolumn{2}{c|}{\begin{sideways}UCB\end{sideways}}
		&\multicolumn{2}{c|}{\begin{sideways}OH-UCB\end{sideways}}
		&\multicolumn{2}{c|}{\begin{sideways}O-UCB\end{sideways}}
		&\multicolumn{2}{c}{\begin{sideways}MultiSBM\end{sideways}}
		\\ 
		
		&\begin{sideways}Value\end{sideways}
		&\begin{sideways}Deaths\end{sideways}
		&\begin{sideways}Value\end{sideways}
		&\begin{sideways}Deaths\end{sideways}
		&\begin{sideways}Value\end{sideways}
		&\begin{sideways}Deaths\end{sideways}
		&\begin{sideways}Value\end{sideways}
		&\begin{sideways}Deaths\end{sideways}
		\\ 
\hline
\hline
0.1
& 0.739 & 42.4 
& 0.699 & 2.6 
& 0.737 & 39.75 
& 0.725 & 55.15 
\\ 
\hline
0.2
& 0.75 & 56 
& 0.698 & \textbf{2.4} 
& 0.754 & 61.3 
& 0.73 & 67.15 
\\
\hline
0.4
& 0.78 & 80.45 
& 0.699 & 4.6 
& \textbf{0.784} & 84.45 
& 0.709 & 67.4 
\\
\hline
0.6
& 0.766 & 78.3 
& 0.696 & 4.75 
& 0.776 & 82.75 
& 0.683 & 72.1 
\\
\hline
0.8
& 0.761 & 75.3 
& 0.695 & 7.6 
& 0.773 & 83.8 
& 0.659 & 82.1 
\\
\hline
1.0
& 0.737 & 73.75 
& 0.693 & 12.8 
& 0.755 & 83.65 
& 0.641 & 89.3 
\\
\hline
1.2
& 0.732 & 71.8 
& 0.695 & 16.6 
& 0.753 & 81.45 
& 0.645 & 87.25 
\\
\hline
1.4
& 0.727 & 69.8 
& 0.689 & 20.55 
& 0.74 & 81.65 
& 0.628 & 95.85 
\\
\hline
1.6
& 0.713 & 70.5 
& 0.686 & 25 
& 0.729 & 82.8 
& 0.624 & 97.25 
\\
\hline
1.8
& 0.703 & 71.5 
& 0.685 & 29.85 
& 0.722 & 82.65 
& 0.623 & 98.6 
\\
\hline
2.0
& 0.7 & 71.45 
& 0.68 & 33.3 
& 0.714 & 82.85 
& 0.617 & 100.85 
\\
\hline
\end{tabular}
\caption{The average wellbeing value and amount of deaths per algorithm and $c$ parameter averaged over 20 runs.}
\label{tbl:banditResults}
\end{table}

\subsection{Tree Search Setup}

We test the three algorithms described above (MCTS, \omcts{} and \pbmcts{}) using the General Video Game AI (GVGAI) framework \cite{perez2016general}.
As additional benchmarks we add \mixmax{} (Q parameter set to $0.25$) as an MCTS variation that was suggested by \cite{khalifa2016modifying} to tackle the cowardly behavior, and \textsc{Yolobot}, a state of the art GVGAI agent \cite{YOLOBOT,perez2018general}.
GVGAI has implemented a variety of different video games and provides playing agents with a unified interface to simulate moves using a forward model.
Using this forward model is expensive so that simulations take a lot of time.
We use the number of calls to this forward model as a computational budget.
In comparison to using the real computation time, it is independent of specific hardware, algorithm implementations, and side effects such as logging data.

Our algorithms are given access to the following pieces of information provided by the framework:
\begin{description}
	\item[\emph{Available actions}:] The actions the agent can perform in a given state
	\item[\emph{Game score}:] The score of the given state $\in \mathbb{N}$. Depending on the game this ranges from $0$ to $1$ or $-1000$ to $10000$.
	\item[\emph{Game result}:] The result of the game: \emph{won}, \emph{lost} or \emph{running}.
	\item[\emph{Simulate action}:] The forward model. It is stochastic, e.g., for enemy moves or random object spawns.
\end{description}
\subsubsection{Heuristic Monte Carlo Tree Search}
The games in GVGAI have a large search space with $5$ actions and up to $2000$ turns.
Using vanilla MCTS, one rollout may use a substantial amount of time, since up to $2000$ moves have to be made to reach a terminal state.
To achieve a good estimate, many rollouts have to be simulated.
Hence it is common to stop rollouts early at non-terminal states, using a heuristic to estimate the value of these states.
In our experiments, we use this variation of MCTS, adding the maximal length for rollouts \textit{RL} as an additional parameter. 
The heuristic value at non-terminal nodes is computed in the same way as the terminal reward (i.e., it essentially corresponds to the score at this state of the game).
\subsubsection{Mapping Rewards to $\mathbb{R}$}
The objective function has two dimensions: on the one hand, the agent needs to win the game by achieving a certain goal, on the other hand, the agent also needs to maximize its score. 
Winning is more important than getting higher scores.

Since MCTS needs its rewards being $\in \mathbb{R}$ or even better $\in [0,1]$, the two-dimensional target function needs to be mapped to one dimension, in our case for comparison and ease of tuning parameters into $[0, 1]$.
%
%
Knowing the possible scores of a game, the score can be normalized by $r_{norm} = (r-r_{min})/(r_{max}-r_{min})$ with $r_{max}$ and $r_{min}$ being the highest and lowest possible score.

For modeling the relation  \emph{lost} $\prec$ \emph{playing} $\prec$ \emph{won} which must hold for all states, we split 
the interval $[0,1]$ into three equal parts 
(cf.~also the axis of Figure~\ref{fig:loss}):
  \begin{equation}
  	r_{mcts} = \frac{r_{norm}}{3} + 
    \begin{cases}
      0, & \text{if \it lost} \\
      \frac{1}{3}, & \text{if \it playing} \\
      \frac{2}{3}, & \text{if \it won}.
    \end{cases}
    \label{eq:rewards}
  \end{equation}
This is only one of many possibilities to map the rewards to $[0,1]$, but it is an obvious and straight-forward approach. 
Naturally, the results for the MCTS techniques, which use this reward, will change when a different reward mapping is used, and their results can probably be improved by shaping the reward. In fact, one of the main points of our work is to show that for 
%
\omcts{} (as well as for \pbmcts{}) no such reward shaping is necessary because these algorithms do not rely on the numerical information. In fact, for them, 
the mapped linear function with $a\succ b \Leftrightarrow r_{mcts}(a) > r_{mcts}(b)$ is equivalent to the preferences induced by the two-dimensional feedback.

\subsubsection{Selected Games}
GVGAI provides users with many games.
Doing an evaluation on all of them is not feasible.
Furthermore, some results would exhibit erratic behavior,
since the tested algorithms (except for \textsc{Yolobot}) are not suitable for solving some of the games.
For example, true rewards often are very sparse, and the agent has to be guided in some way to reliably 
solve the game.

For this reason, we manually played all the games and selected a variety of interesting, and not too complex games with different characteristics, which we believed to be solvable for the tested algorithms:
\begin{itemize}
	\item \game{Zelda}: The agent can hunt monsters and slay them with its sword. It wins by finding the key and taking the door.
	\item \game{Chase}: The agent has to catch all animals which flee from the agent. Once an animal finds a caught one, it gets angry and chases the agent. If the agent get caught this way, the game is lost.
	\item \game{Whackamole}: The agent can collect mushrooms which spawn randomly. A cat helps it in doing so. The game is won after 2000 time steps or lost if agent and cat collide.
	\item \game{Boulderchase}: The agent can dig through sand to a door that opens after it has collected ten diamonds. Monsters chase it through the sand turning sand into diamonds. It may be very hard for a MCTS agent to solve this game.
	\item \game{Surround}: The agent can win the game at any time by taking a specific action, or collect points by moving while leaving a snake-like trail. A moving enemy also leaves a trail. 
	The agent shall not collide with trails.
	\item \game{Jaws}: The agent controls a submarine, which is hunted by a shark. It can shoot fish giving points and leaving an item behind. Once 20 items are collected, a collision with the shark gives a large number of points, 
	otherwise it
	 loses the game. Colliding with fish always loses the game. The fish spawn randomly on 6 specific positions.
	\item \game{Aliens}: The agent can only move from left to right and shoot upwards. Aliens come flying from top to bottom throwing rocks on the agent. For increasing the score, the agent can shoot the aliens or shoot disappearing blocks.
\end{itemize}

The number of iterations that can be performed by the algorithms depends
on the computational budget of calls to the forward model.
We tested the algorithms with $250$, $500$, $1000$ and $10000$ forward model uses (later called \emph{budget}). Thus, in total, we experimented with $28$ problem settings ($7$ domains $\times$ $4$ different budgets).

\subsubsection{Tuning Algorithms and Experiments}
All MCTS algorithms have two parameters in common, the
\emph{exploration trade-off} $\Cp$ and \emph{rollout length} $RL$.
For $RL$ we tested 4 different values: $5, 10, 25$ and $50$, and
for $\Cp$ we tested 9 values from $0$ to $2$ in steps of size $0.25$.
In total, these are 36 configurations per algorithm.
To reduce variance, we have repeated each experiment 40 times.
Overall,  
4 algorithms with 36 configurations were run 40 times on 28 problems, 
resulting in 161280 games played for tuning.

Additionally, we compare the algorithms to \textsc{Yolobot}, a highly competitive GVGAI agent that won several challenges \cite{YOLOBOT,perez2018general}.
\textsc{Yolobot} is able to solve games none of the other five algorithms can solve.
Note that \textsc{Yolobot} is designed and tuned to act within a 20ms time limit.
Scaling and even increasing budget might lead to worse and unexpected behavior.
Still it is added for sake of comparison and interpretability of strength.
For \textsc{Yolobot} each of the $28$ problems is played $40$ times, which leads to $1120$ additional games or $162400$ games in total.\footnote{Results are available at \url{https://github.com/Muffty/OMCTS_Appendix}}

We are mainly interested on 
how well the different algorithms perform on the problems, given optimal tuning per problem.
To give an answer, we show the performance of the algorithms per problem in percentage of wins and obtained average score.
We do a Friedmann test on average ranks of those data with a posthoc Wilcoxon signed rank test to test for significance \cite{demvsar2006statistical}.
Additionally, we show and discuss the performance of all parameter configurations.

\begin{table}[htbp]
\centering
\small
\tabcolsep=0.11cm
	\begin{tabular}{cc||r|r|r|r|r}
		\begin{sideways}Game\end{sideways}&\begin{sideways}Time\end{sideways}
		&\multicolumn{1}{c|}{\begin{sideways}\smaller \omcts{}\end{sideways}}
		&\multicolumn{1}{c|}{\begin{sideways}\smaller MCTS\end{sideways}}
		&\multicolumn{1}{c|}{\begin{sideways}\smaller \textsc{Yolo}-\end{sideways} \begin{sideways}\smaller \textsc{bot}\end{sideways}}
		&\multicolumn{1}{c|}{\begin{sideways}\smaller \pbmcts{}\end{sideways}}
		&\multicolumn{1}{c}{\begin{sideways}\smaller \mixmax{} \end{sideways}}
		\\ 
\hline
\hline
\multirow{8}{*}{\begin{sideways}Jaws\end{sideways}}
&\multirow{ 2}{*}{$10^4$}
&\textbf{     100\%}&\textbf{     100\%}&      27.5\%&      80\%&      67.5\% \\ 
 &&\textbf{    1083.8}&     832.7&     274.7&     895.7&     866.8\\\cline{2-7}
&\multirow{ 2}{*}{$10^3$}
&      92.5\%&\textbf{      95\%}&      35\%&      52.5\%&      65\% \\ 
 &&\textbf{    1028.2}&     958.9&     391.0&     788.5&     736.4\\\cline{2-7}
&\multirow{ 2}{*}{500}
&      85\%&\textbf{      90\%}&      65\%&      50\%&      52.5\% \\ 
 &&     923.4&\textbf{    1023.1}&     705.7&     577.6&     629.0\\\cline{2-7}
&\multirow{ 2}{*}{250}
&\textbf{      85\%}&\textbf{      85\%}&      32.5\%&      37.5\%&      37.5\% \\ 
 &&\textbf{    1000.9}&     997.6&     359.6&     548.8&     469.0\\
\hline
\hline
\multirow{8}{*}{\begin{sideways}Surround\end{sideways}}
&\multirow{ 2}{*}{$10^4$}
&\textbf{     100\%}&\textbf{     100\%}&\textbf{     100\%}&\textbf{     100\%}&\textbf{     100\%} \\ 
 &&\textbf{      81.5}&      71.0&      81.2&      64.3&      57.6\\\cline{2-7}
&\multirow{ 2}{*}{$10^3$}
&\textbf{     100\%}&\textbf{     100\%}&\textbf{     100\%}&\textbf{     100\%}&\textbf{     100\%} \\ 
 &&\textbf{      83.0}&      80.8&      77.3&      40.8&      25.0\\\cline{2-7}
&\multirow{ 2}{*}{500}
&\textbf{     100\%}&\textbf{     100\%}&\textbf{     100\%}&\textbf{     100\%}&\textbf{     100\%} \\ 
 &&\textbf{      84.6}&      61.8&      83.3&      26.3&      17.3\\\cline{2-7}
&\multirow{ 2}{*}{250}
&\textbf{     100\%}&\textbf{     100\%}&\textbf{     100\%}&\textbf{     100\%}&\textbf{     100\%} \\ 
 &&\textbf{      83.4}&      64.7&      76.1&      14.3&      10.3\\
\hline
\hline
\multirow{8}{*}{\begin{sideways}Aliens\end{sideways}}
&\multirow{ 2}{*}{$10^4$}
&\textbf{     100\%}&\textbf{     100\%}&\textbf{     100\%}&\textbf{     100\%}&\textbf{     100\%} \\ 
 &&\textbf{      82.4}&      81.6&      81.5&      81.8&      77.0\\\cline{2-7}
&\multirow{ 2}{*}{$10^3$}
&\textbf{     100\%}&\textbf{     100\%}&\textbf{     100\%}&\textbf{     100\%}&\textbf{     100\%} \\ 
 &&      79.7&      78.4&\textbf{      82.2}&      76.9&      76.4\\\cline{2-7}
&\multirow{ 2}{*}{500}
&\textbf{     100\%}&\textbf{     100\%}&\textbf{     100\%}&\textbf{     100\%}&\textbf{     100\%} \\ 
 &&      78.0&      77.3&\textbf{      81.1}&      77.2&      76.0\\\cline{2-7}
&\multirow{ 2}{*}{250}
&\textbf{     100\%}&\textbf{     100\%}&\textbf{     100\%}&\textbf{     100\%}&\textbf{     100\%} \\ 
 &&      77.7&      77.1&\textbf{      79.3}&      75.8&      74.8\\
\hline
\hline
\multirow{8}{*}{\begin{sideways}Chase\end{sideways}}
&\multirow{ 2}{*}{$10^4$}
&\textbf{      87.5\%}&      80\%&      50\%&      67.5\%&      37.5\% \\ 
 &&\textbf{       6.2}&       6.0&       4.8&       5.2&       3.9\\\cline{2-7}
&\multirow{ 2}{*}{$10^3$}
&      60\%&      50\%&\textbf{      70\%}&      30\%&      17.5\% \\ 
 &&       4.8&       4.8&\textbf{       5.1}&       3.7&       2.6\\\cline{2-7}
&\multirow{ 2}{*}{500}
&      55\%&      45\%&\textbf{      90\%}&      27.5\%&      12.5\% \\ 
 &&       4.9&       4.5&\textbf{       5.5}&       2.9&       2.1\\\cline{2-7}
&\multirow{ 2}{*}{250}
&      40\%&      32.5\%&\textbf{      90\%}&      17.5\%&       7.5\% \\ 
 &&       4.2&       4.1&\textbf{       5.6}&       2.5&       2.6\\
\hline
\hline
\multirow{8}{*}{\begin{sideways}Boulderchase\end{sideways}}
&\multirow{ 2}{*}{$10^4$}
&      62.5\%&      75\%&      45\%&\textbf{      82.5\%}&      30\% \\ 
 &&      23.7&      22.1&      18.8&\textbf{      27.3}&      20.1\\\cline{2-7}
&\multirow{ 2}{*}{$10^3$}
&      50\%&      32.5\%&\textbf{      52.5\%}&      40\%&      22.5\% \\ 
 &&\textbf{      22.8}&      18.6&      21.8&      18.1&      16.2\\\cline{2-7}
&\multirow{ 2}{*}{500}
&\textbf{      47.5\%}&      30\%&      35\%&      32.5\%&      15\% \\ 
 &&\textbf{      24.7}&      20.2&      18.3&      19.4&      14.4\\\cline{2-7}
&\multirow{ 2}{*}{250}
&      40\%&      40\%&\textbf{      60\%}&      17.5\%&      15\% \\ 
 &&      20.9&      20.1&\textbf{      21.7}&      14.7&      15.3\\
\hline
\hline
\multirow{8}{*}{\begin{sideways}Whackamole\end{sideways}}
&\multirow{ 2}{*}{$10^4$}
&\textbf{     100\%}&\textbf{     100\%}&      75\%&      97.5\%&      75\% \\ 
 &&\textbf{      72.5}&      44.4&      37.0&      60.1&      48.5\\\cline{2-7}
&\multirow{ 2}{*}{$10^3$}
&\textbf{     100\%}&\textbf{     100\%}&      55\%&      77.5\%&      65\% \\ 
 &&\textbf{      64.0}&      41.8&      33.9&      43.9&      39.0\\\cline{2-7}
&\multirow{ 2}{*}{500}
&\textbf{     100\%}&\textbf{     100\%}&      57.5\%&      70\%&      52.5\% \\ 
 &&\textbf{      59.5}&      50.0&      29.0&      38.1&      35.4\\\cline{2-7}
&\multirow{ 2}{*}{250}
&      97.5\%&\textbf{     100\%}&      50\%&      65\%&      52.5\% \\ 
 &&\textbf{      54.8}&      45.9&      28.5&      35.1&      26.6\\
\hline
\hline
\multirow{8}{*}{\begin{sideways}Zelda\end{sideways}}
&\multirow{ 2}{*}{$10^4$}
&\textbf{      97.5\%}&      87.5\%&      95\%&      90\%&      70\% \\ 
 &&       8.3&       7.4&       3.8&\textbf{       9.6}&       8.1\\\cline{2-7}
&\multirow{ 2}{*}{$10^3$}
&      80\%&      85\%&\textbf{      87.5\%}&      57.5\%&      42.5\% \\ 
 &&\textbf{       8.8}&       7.5&       5.3&       8.6&\textbf{       8.8}\\\cline{2-7}
&\multirow{ 2}{*}{500}
&      62.5\%&      75\%&\textbf{      77.5\%}&      50\%&      35\% \\ 
 &&       8.6&       8.2&       4.6&\textbf{       8.8}&       7.8\\\cline{2-7}
&\multirow{ 2}{*}{250}
&      55\%&      55\%&\textbf{      70\%}&      45\%&      30\% \\ 
 &&\textbf{       8.4}&       7.8&       4.4&       8.0&       7.2\\
\hline
\hline
\multicolumn{2}{c||}{$\emptyset$ Rank}
&       1.6
&       2.5
&       2.6
&       3.5
&       4.7
\end{tabular}
\caption{The results of algorithms tuned per row.}
\label{tbl:reslt}
\end{table}

\subsection{Tree Search Results}
\label{sec:Experiments}
%
Table~\ref{tbl:reslt} shows the best win rate and the corresponding average score of each algorithm, averaged over $40$ runs
for each of the $36$ different parameter settings.
In each row, the best values for the win rate and the average score are shown in bold, and a ranking of the algorithms is computed. The resulting average ranks 
are shown in the last line. 
We use 
a Friedmann test and a posthoc Wilcoxon signed rank test as an indication for significant differences in performance.
The results of the latter (with a significance level of $99\%$) are shown in Figure~\ref{fig:wilcoxonAll}.

\begin{figure}[htb]
	\centering
	\begin{subfigure}{\columnwidth}
		\centering
		\includegraphics[width=1\columnwidth]{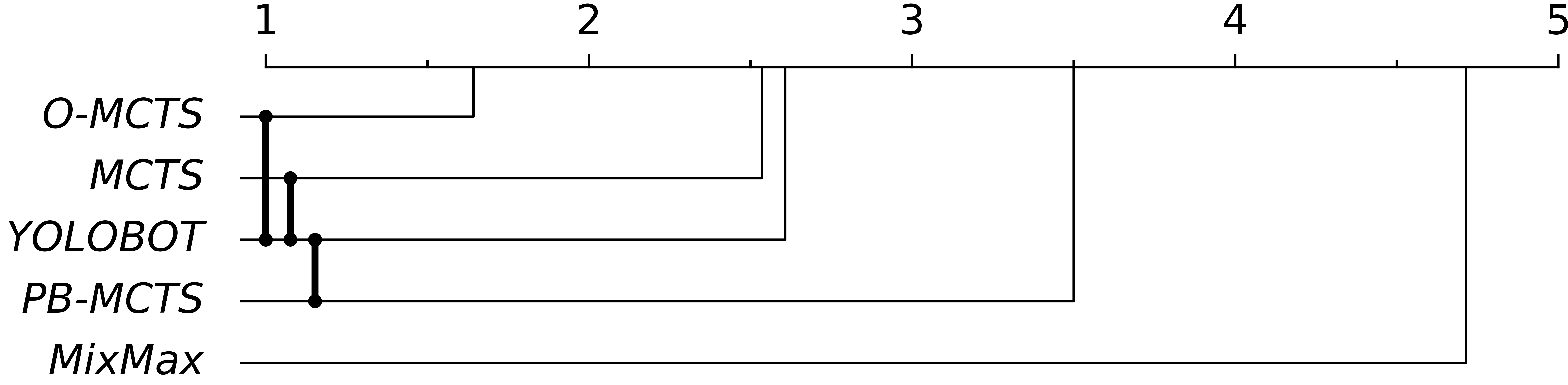}
		\caption{All game runs. Data from Table~\ref{tbl:reslt}}
		\label{fig:wilcoxonAll}
	\end{subfigure}
	
	\begin{subfigure}{\columnwidth}
		\centering
		\includegraphics[width=1\columnwidth]{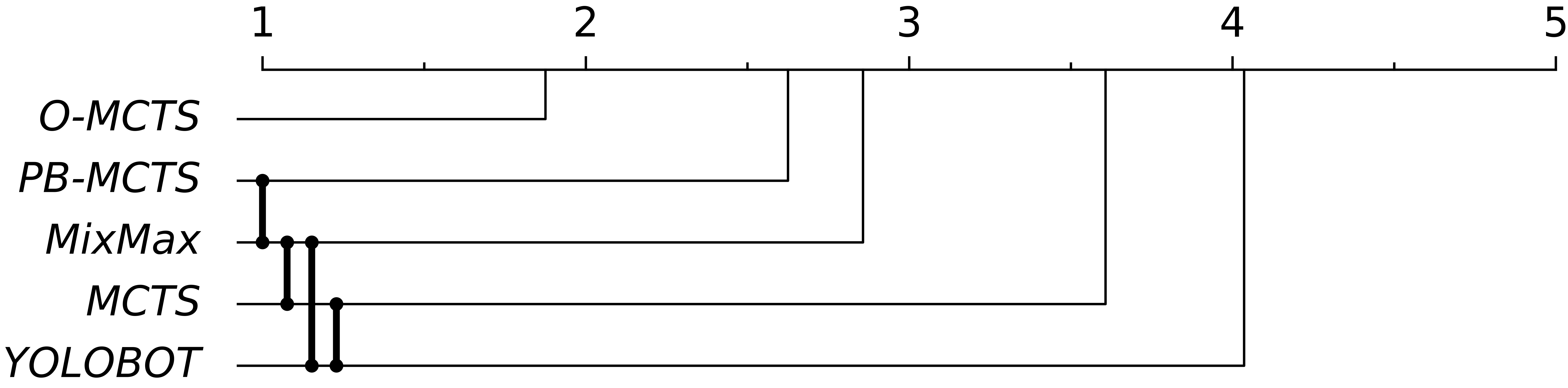}
		\caption{Only won game runs}
		\label{fig:wilcoxonWon}
	\end{subfigure}
	\caption{Average ranks and the result of a Wilcoxon signed rank test with $\alpha=0.01$. Directly connected algorithms do not differ significantly.}
\end{figure}

\begin{table}[t!]
        \centering\includegraphics[width=0.98\columnwidth]{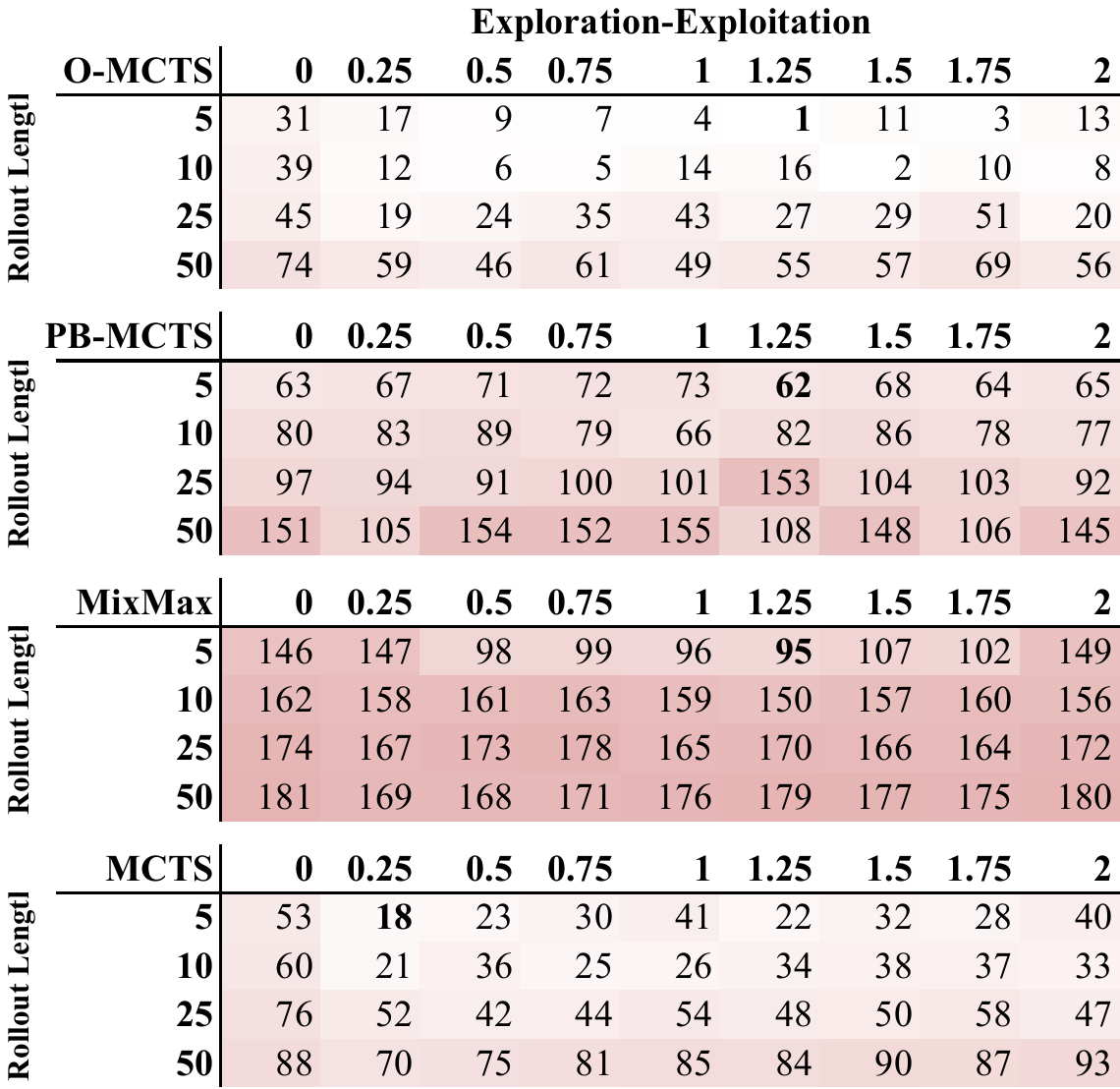}
\caption{Results for different parameters for all algorithms except of \Yolobot{} (Rank 15). In each cell, the overall rank over all games and budgets is shown. 
}
\label{tbl:reslt2}
\end{table}

We can see that \omcts{} performed best with an average rank of $1.6$ and a significantly better performance than MCTS and \pbmcts{}.
Table~\ref{tbl:reslt} allows us to take a closer look on the domains.
For games that are easy to win, such as \game{Surround}, \game{Aliens}, and \game{Whackamole} \omcts{} beats MCTS and \pbmcts{} by winning with a higher score.
In \game{Chase}, a deadly but more deterministic game, \omcts{} is able to achieve a higher win rate.
In deadly and stochastic games like \game{Zelda}, \game{Boulderchase} and \game{Jaws} \omcts{} performs comparable to the other algorithms without anyone performing significant better than the others.


Figure~\ref{fig:wilcoxonWon} summarizes the results when only won games are considered. 
It can be seen, that in this case, \pbmcts{} is significantly better than MCTS.
This implies that if \pbmcts{} manages to win, it does so with a greater score than MCTS, but it wins less often.
\Yolobot{} falls behind because it is designed to primarily maximize the win rate, not the score. 

Inspecting the performance of \mixmax{} it can easily be seen that the hereby added bias towards higher scores often results in a death: Looking at only won games (see Figure~\ref{fig:wilcoxonWon}) it achieves a higher rank than MCTS, but overall its performance is significantly worse. 

In conclusion, we found evidence that \omcts{}'s preference for actions that \emph{maximize win rate} works better than MCTS's tendency to \emph{maximize average performance} for the tested domains.

\paragraph{Parameter Optimization}
In Table~\ref{tbl:reslt2} the overall rank over all parameters for all algorithms are shown.
It is clearly visible that a low rollout length $RL$ improves performance and is more important to tune correctly than the exploration-exploitation trade-off $\Cp$.
Since \Yolobot{} has no parameters, it is not shown.
Except for the extreme case of no exploration ($\Cp=0$), \omcts{} with $RL=5$ is better than any other MCTS algorithm.
The best configuration is \omcts{} with $RL=5$ and $\Cp=1.25$.

\paragraph{Video Demonstrations}
For each algorithm and game, we have recorded a video where the agent wins\footnote{You can watch the videos at \url{https://bit.ly/2ohbYb3}}.
In those videos it can be seen that \omcts{} frequently plays actions that lead to a higher score, whereas 
MCTS play more safely---often too cautious and averse to risking any potentially deadly effect.


\section{Conclusion}
\label{sec:Conclusion}
In this paper we proposed \omcts{}, a modification of MCTS that handles the rewards in an ordinal way:
Instead of averaging backpropagated values to obtain a value estimation, it estimates the winning probability of an action using the Borda score. 
By doing so, the magnitude of distances between different reward signals are disregarded, which can be useful in ordinal domains.
In our experiments 
using
the GVGAI framework, we compared \omcts{} to MCTS, \pbmcts{}, \mixmax{} and \Yolobot{}, a specialized agent for this domain.
Overall, \omcts{} achieved higher win rates and reached higher scores than the other algorithms, confirming that this approach can even be useful in domains where numeric reward information is available.
\omcts{} is based on \omab, an ordinal variation of \ucb{} which we also introduced and tested.
Additionally, we have introduced a hierarchical version of \omab{} with which it is possible to define ordinal thresholds that are optimized first.

\paragraph{Acknowledgments}
This work was supported by the German Research Foundation (DFG project number FU 580/10). 
We gratefully acknowledge the use of the Lichtenberg high performance computer of the TU Darmstadt for our experiments.

\bibliographystyle{spmpsci}      
\bibliography{mybib}
\end{document}